\documentclass{article}

\usepackage{times}
\usepackage{graphicx} 

\usepackage{natbib}

\usepackage{algorithm}
\usepackage{algorithmic}

\usepackage{hyperref}

\usepackage[accepted]{icml2017}

\usepackage{color,graphicx,amssymb,amsmath}
\usepackage{amsthm}
\usepackage{etoolbox}
\usepackage{bm}
\usepackage{mathtools}
\usepackage{algorithm}
\usepackage{algorithmic}
\usepackage{caption} 
\usepackage{subcaption}
\usepackage{xspace}
\usepackage{multirow}
\usepackage[capitalise]{cleveref}
\usepackage{array}
\usepackage{nicefrac}
\usepackage{url}

\newcommand{\tabincell}[2]{\begin{tabular}{@{}#1@{}}#2\end{tabular}}
\usepackage{threeparttable} 
\usepackage{paralist} 

\renewcommand{\algorithmicrequire}{\textbf{Input:}}
\renewcommand{\algorithmicensure}{\textbf{Output:}}

\newcommand{\MyMapTemplatePrefixc}[4]{\expandafter#1\csname#3#4\endcsname{#2{#4}}} 
\forcsvlist{\MyMapTemplatePrefixc {\def} {\mathcal}{mc}} {A,B,C,D,E,F,G,H,I,J,K,L,M,N,O,P,Q,R,S,T,U,V,W,X,Y,Z}  

\def\tz{\tilde{z}}

\def\ga{\alpha}
\def\gb{\beta}

\def\gep{\epsilon}

\def\gl{\lambda}

\def\gD{\Delta}

\def\gS{\Sigma}

\def\hga{\hat{\alpha}}
\def\hgb{\hat{\beta}}
\def\hgl{\hat{\lambda}}

\def\bbR{{\mathbb R}}

\def\st{\mbox{subject to}}

\newcommand{\trans}{^{T}}
\newcommand{\inv}{^{-1}}
\newcommand{\pinv}{^{\dagger}}
\newcommand{\inprod}[1]{\langle #1 \rangle}

\newcommand{\itk}{^{k}}
\newcommand{\ito}{^{k_0}}
\newcommand{\kp}{^{k+1}}
\newcommand{\km}{^{k-1}}
\newcommand{\subi}{_{i}}
\newcommand{\subk}{_{k}}
\newcommand{\subkp}{_{k+1}}
\newcommand{\opt}{^{*}}
\newcommand{\reso}[1]{\gD #1 ^* }
\newcommand{\resp}[1]{\gD #1 ^+ }

\newtheorem{proposition}{Proposition}
\crefname{proposition}{Proposition}{Proposition}
\newtheorem{theorem}{Theorem}
\crefname{theorem}{Theorem}{Theorem}
\newtheorem{lemma}{Lemma}
\crefname{lemma}{Lemma}{Lemma}
\crefname{myrl}{Rule}{Rule}
\newtheorem{assumption}{Assumption}
\crefname{assumption}{Assumption}{Assumption}

\icmltitlerunning{
Adaptive Consensus ADMM for Distributed Optimization
}

\usepackage{color}

\begin{document} 

\twocolumn[
\icmltitle{
Adaptive Consensus ADMM for Distributed Optimization
}




\begin{icmlauthorlist}
\icmlauthor{Zheng Xu}{umd}
\icmlauthor{Gavin Taylor}{usna}
\icmlauthor{Hao Li}{umd}
\icmlauthor{M\'{a}rio A. T. Figueiredo}{ist}
\icmlauthor{Xiaoming Yuan}{hkb}
\icmlauthor{Tom Goldstein}{umd}
\end{icmlauthorlist}


\icmlaffiliation{umd}{University of Maryland, College Park;}
\icmlaffiliation{usna}{United States Naval Academy, Annapolis;}
\icmlaffiliation{ist}{Instituto de Telecomunica\c{c}\~{o}es, IST, ULisboa, Portugal;}
\icmlaffiliation{hkb}{Hong Kong Baptist University, Hong Kong}

\icmlcorrespondingauthor{Zheng Xu}{xuzhustc@gmail.com}

\icmlkeywords{ADMM, distributed computing, convergence rate, adaptive stepsize}

\vskip 0.3in
]



\printAffiliationsAndNotice{}  

\begin{abstract} 
The alternating direction method of multipliers (ADMM) is commonly used for distributed model fitting problems, but its performance and reliability depend strongly on user-defined penalty parameters.   We study distributed ADMM methods that boost performance by using different fine-tuned algorithm parameters on each worker node. We present a $O(1/k)$ convergence rate for adaptive ADMM methods with node-specific parameters, and propose \emph{adaptive consensus ADMM (ACADMM)},
which automatically tunes parameters without user oversight.
\end{abstract} 

\section{Introduction}
The alternating direction method of multipliers (ADMM) is a popular tool for solving problems of the form,
\begin{eqnarray}
\min_{u\in \bbR^n,v\in \bbR^m}  f(u) + g(v),~~~~\st~~  Au+Bv = b, \label{eq:prob}
\end{eqnarray}
where $f:\bbR^n\rightarrow \bbR$ and $g:\bbR^m\rightarrow \bbR$ are convex functions, $A\in \bbR^{p \times n}$, $B\in \bbR^{p \times m}$, and $b \in \bbR^p$.
ADMM was first introduced in \cite{glowinski1975approximation} and
\cite{gabay1976dual}, and has found applications in many optimization problems
in machine learning, distributed computing and many other areas
\cite{boyd2011admm}.

Consensus ADMM \cite{boyd2011admm} solves minimization problems involving a
composite objective $f(v)=\sum_i f_i(v),$ where worker $i$ stores the data
needed to compute $f_i,$ and so is well suited for distributed model fitting
problems \cite{
boyd2011admm,
zhang2014asynchronous,
song2015fast,
chang2016asynchronous,
goldstein2016unwrapping,taylor2016training}. To distribute this problem, consensus methods
assign a separate copy of the unknowns, $u_i,$ to each worker, and then
apply ADMM to solve
\begin{align}
\min_{u_i\in \bbR^d, v\in \bbR^d} \  \sum_{i=1}^N f_i(u_i) + g(v),~~~~\st~~ u_i = v, \label{eq:cprob}
\end{align}
where $v$ is the ``central'' copy of the unknowns,  and $g(v)$ is a
regularizer. The consensus problem \eqref{eq:cprob} coincides with \eqref{eq:prob} by defining $u=(u_1;\, \ldots; \, u_N)\in \bbR^{dN}$, $A=I_{dN}\in \bbR^{dN \times dN}$, and $B = -(I_d; \,\ldots;\,I_d) \in \bbR^{dN \times d}$, where $I_d$ represents the $d\times d$ identity matrix.

ADMM methods rely on a penalty parameter (stepsize) that is chosen by the user. In theory, ADMM converges for any constant penalty parameter
\cite{eckstein1992douglas,he2012con,ouyang2013stochastic}. In
practice, however, the efficiency of ADMM is highly sensitive to this
parameter choice
\cite{nishihara2015general,ghadimi2015optimal}, and
can be improved via adaptive penalty selection methods
\cite{he2000alternating,song2015fast,xu2016adaptive}. 

One such approach, residual balancing (RB) \cite{he2000alternating}, adapts
the penalty parameter so that the residuals (derivatives of the
Lagrangian with respect to primal and dual variables) have similar magnitudes.
When the same penalty parameter is used across nodes, RB is known to converge, although without a known rate guarantee.  A more recent approach, AADMM \cite{xu2016adaptive},
achieves impressive practical convergence speed on many applications, including
consensus problems, with adaptive penalty parameters by estimating the local
curvature of the dual functions. However, the dimension of the unknown variables
in consensus problems grows with the number of distributed nodes, causing the
curvature estimation to be inaccurate and unstable. AADMM uses the same convergence analysis as RB. Consensus residual balancing (CRB) \cite{song2015fast} extends residual balancing to consensus-based
ADMM for distributed optimization by balancing the local primal and dual
residuals on each node.  However, convergence guarantees for this method are fairly weak, and adaptive
penalties need to be reset after several iterations to guarantee convergence.

%

We study the use of adaptive ADMM in the distributed setting, where different workers use different local algorithm parameters to accelerate convergence.  We begin by studying the theory and provide convergence guarantees when node-specific penalty parameters are used. We demonstrate a $O(1/k)$ convergence rate under mild conditions that is applicable for many forms of adaptive ADMM including all the above methods. Our theory is more general than the convergence guarantee in \cite{he2000alternating,xu2016adaptive} that only shows convergence when the scalar penalty parameter is adapted.
Next, we propose an adaptive consensus ADMM (ACADMM) method to automate local algorithm parameters selection. Instead of estimating one global penalty parameter for all workers,  different local penalty parameters are estimated using the local curvature of subproblems on each node. 

\section{Related work}
ADMM is known to have a $O(1/k)$ convergence rate under mild conditions
for convex problems \cite{he2012con,he2015non}, while a $O(1/k^2)$ rate is
possible when at least one of the functions is strongly convex or smooth
\cite{goldfarb2013fast,goldstein2014fast,kadkhodaie2015accelerated,tian2016faster}.
Linear convergence can be achieved with strong convexity assumptions
\cite{davis2014faster,nishihara2015general,giselsson2016linear}.  All of these
results assume constant parameters; to the best of our knowledge, no
convergence rate has been proven for ADMM with an adaptive penalty:
\citep{he2000alternating,xu2017adaptive} proves convergence without providing a rate, and \cite{lin2011linearized,banert2016fixing,goldstein2015adaptive} prove
convergence for some particular variants of ADMM
(``linearized'' or ``preconditioned'').

To improve practical convergence of ADMM, 
fixed optimal parameters are discussed in
\cite{raghunathan2014alternating,ghadimi2015optimal,nishihara2015general,francca2016explicit}.
These methods make strong assumptions about the objective and require
information about the spectrum of $A$ and/or $B$. 
Additionally, adaptive methods have been proposed;
the most closely
related work to our own is \citep{song2015fast}, which extends the results of
\cite{he2000alternating} to consensus problems, where communication is
controlled by predefined network structure and the regularizer $g(v)$ is
absent. In contrast to these methods, the proposed ACADMM extends the spectral
penalty in \cite{xu2016adaptive} to consensus problems and provides
convergence theory that can be applied to a broad range of adaptive ADMM
variants.

\section{Consensus ADMM}
In the following, we use the subscript $i$ to denote iterates computed on the $i$th node, superscript
$k$ is the iteration number, $\gl\subi\itk$ is the dual vector of Lagrange
multipliers, and $\{\tau\itk\subi\}$ are iteration/worker-specific penalty
parameters (contrasted with the single constant penalty parameter $\tau$ of
``vanilla" ADMM).
Consensus methods apply ADMM to \eqref{eq:cprob}, resulting in the steps
\begin{align}
u\subi\kp & =  \arg\min_{u\subi} f\subi(u\subi) + \frac{\tau\subi\itk}{2} \| v\itk -u\subi  + \frac{\gl\subi\itk}{\tau\subi\itk} \|^2 \label{eq:cadmms} \\
v\kp & =  \arg\min_{v} g(v) + \sum_{i=1}^{N} \frac{\tau\subi\itk}{2} \| v - u\subi\kp  + \frac{\gl\subi\itk}{ \tau\subi\itk} \|^2 \label{eq:cadmm2} \\
\gl\subi\kp & = \gl\subi\itk + \tau\itk\subi (v\kp -u\subi\kp ). \label{eq:cadmme} 
\end{align}

The primal and dual residuals, $r^k$ and $d^k,$ are used to monitor
convergence.
{\small%
\begin{equation}
r\itk = 
\begin{pmatrix}
r_{1}\itk \\
\vdots \\
r_{N}\itk
\end{pmatrix}, 
\  d\itk = 
\begin{pmatrix}
d_{1}\itk \\
\vdots \\
d_{N}\itk
\end{pmatrix},
~\ \begin{cases}
r\subi\itk = v\itk - u\subi\itk \\
d\subi\itk = \tau\subi\itk(v\km-v\itk).
\end{cases} \label{eq:res}
\end{equation}
}%
The primal residual $r^k$ approaches zero when the iterates accurately satisfy
the linear constraints in \eqref{eq:cprob}, and the dual residual $d^k$
approaches zero as the iterates near a minimizer of the objective.
Iteration can be terminated when 
\begin{equation}
\begin{split}
& \| r\itk \|^2  \leq \gep^{tol} \max\{\sum\nolimits_{i=1}^{N} \| u\subi\itk\|^2, N\|v\itk\|^2\} \\
& \text{and  } \ \| d\itk \|^2  \leq \gep^{tol} \sum\nolimits_{i=1}^{N} \| \gl\subi\itk\|^2,
\end{split} \label{eq:stop}
\end{equation}
where $\gep^{tol}$ is the stopping tolerance. The residuals in \eqref{eq:res} and stopping criterion in \eqref{eq:stop} are adopted from the general problem \cite{boyd2011admm} to the consensus problem. The observation that residuals $r\itk, d\itk$ can be decomposed into ``local residuals'' $r\subi\itk, d\subi\itk$ has been exploited to generalize the residual balancing method \cite{he2000alternating} for distributed consensus problems \cite{song2015fast}.

\section{Convergence analysis} 
We now study the convergence of ADMM with node-specific adaptive penalty parameters.  We provide conditions on penalty parameters that guarantee convergence, and also a convergence rate.  The issue of how to automatically tune penalty parameters effectively will be discussed in Section \ref{sec:adapt}.
\label{sec:theory}
\subsection{Diagonal penalty parameters for ADMM}
Let $T\itk = \text{diag}(\tau_1\itk I_d,\ldots, \tau_N\itk I_d)$ be a diagonal matrix containing non-negative penalty parameters on iteration $k$. Define the norm 
$\| u \|^2_T = u\trans T u$. 
Using the notation defined above with $u=(u_1;\, \ldots;\, u_N) \in \bbR^{dN},$ we can rewrite the consensus ADMM steps \eqref{eq:cadmms}--\eqref{eq:cadmme} as
{%
\begin{equation}
\begin{split}
u\kp =  & \arg\min_{u} f(u)  + \inprod{-Au, \, \gl\itk}
\\ &\qquad 
+ \nicefrac{1}{2} \| b -  A u - B v\itk \|_{T\itk}^2 
\end{split}\label{eq:gadmm1}
\end{equation}
\begin{equation}
\begin{split}
v\kp =  & \arg\min_{v} g(v)  +\inprod{-Bv, \, \gl\itk} \\
 &\qquad + \nicefrac{1}{2} \| b -  A u\kp - B v  \|_{T\itk}^2 
 \end{split}\label{eq:gadmm2}
 \end{equation}
 \begin{equation}
\gl\kp  = \gl\itk + T\itk (b-Au\kp-Bv\kp). \label{eq:gadmm3}
\end{equation}
}%

When using a diagonal penalty matrix, the generalized residuals become
\begin{equation}
 \begin{cases}
r\itk = b - Au\itk - Bu\itk \\
d\itk =  A\trans T\itk B(v\itk-v\km).
\end{cases} \label{eq:gres}
\end{equation}

The sequel contains a convergence proof for generalized ADMM with adaptive penalty matrix
$T\itk$. Our proof is inspired by the variational inequality (VI) approach in
\cite{he2000alternating,he2012con,he2015non}.

\subsection{Preliminaries}
\label{sec:pre}
\textbf{Notation.}
We use the following notation to simplify the discussions.
Define the combined variables
$y = (u ;v) \in \bbR^{n+m}$ and $z= (u;v;\gl)\in \bbR^{n+m+p}$, and denote iterates as $y\itk = ( u\itk ; v\itk )$ and $z\itk = (u\itk; v\itk; \gl\itk).$ Let $y\opt$ and $z\opt$ denote optimal primal/dual solutions. Further define $\resp{z\subk}=(\resp{u\subk}; \resp{v\subk}; \resp{\gl\subk}) :=z\kp-z\itk$ and $\reso{z\subk}=(\reso{u\subk}; \reso{v\subk}; \reso{\gl\subk}) :=z\opt-z\itk$.   Set
$$\phi(y) = f(u) + g(v), \,\, 
F(z) = \left(\begin{array}{c}
-A\trans \gl \\
 -B\trans \gl \\ 
 A u + B v - b
 \end{array}\right),$$
$$H\itk \! =  \!
\begin{pmatrix}
0&0& 0\\
0 &\!\!\! B\trans T\itk B \!\!\!\!\!\!\!\!& 0 \\
0 & 0 & (T\itk)\inv  
 \end{pmatrix}
 \!\!,\,
 M\itk \!= \!
\begin{pmatrix}
I_n & 0 & 0 \\
0 & I_m &0 \\
0 & -T\itk B & I_p 
\end{pmatrix}\!\!.
$$
Note that $F(z)$ is a monotone operator satisfying $\forall z, z', (z-z')\trans(F(z)-F(z')) \geq 0 $.
We introduce intermediate variable $\tz\kp = (u\kp; v\kp; \hat\gl\kp)$, where $\hat\gl\kp = \gl\itk + T\itk (b-Au\kp-Bv\itk)$. We thus have
\begin{equation} \label{eq:zmtz}
\begin{split}
\resp{z\subk} & = M\itk (\tz\kp - z\itk).
\end{split}
\end{equation}

\textbf{Variational inequality formulation.}
The optimal solution $z\opt$ of problem \eqref{eq:prob} satisfies the variational inequality (VI),
\begin{equation}
\forall z, \,  \phi(y)-\phi(y\opt) + (z-z\opt)\trans F(z\opt) \geq 0. \label{eq:optvi}
\end{equation}
From the optimality conditions for the sub-steps (\ref{eq:gadmm1}, \ref{eq:gadmm2}), we see that $y\kp$ satisfies the variational inequalities
\begin{align}
\begin{split}
\forall u, \, & f(u) - f(u\kp) + (u-u\kp)\trans  \\
& \  ( A \trans T\itk (Au\kp + Bv\itk -b) - A\trans \gl\itk) \geq 0
\end{split} \label{eq:hvi}\\
\begin{split}
\forall v, \, & g(v) - g(v\kp) + (v-v\kp)\trans \\
& \ ( B \trans T\itk (Au\kp + Bv\kp -b) - B\trans \gl\itk) \geq 0,
\end{split}\label{eq:gvi}
\end{align}
which can be combined as
%
%
{%
\begin{multline}
 \phi(y)-\phi(y\kp)\\ + (z-\tz\kp)\trans \left( F(\tz\kp) + H\itk \resp{z\subk} \right) \geq 0.  \label{eq:kpvi2}
\end{multline}
}%

\label{sec:lemmas}
\textbf{Lemmas.}
We present several lemmas to facilitate the proof
  of our main convergence theory, which extend previous results regarding ADMM
  \cite{he2012con,he2015non} to ADMM with a diagonal penalty matrix.  \cref{lm3}
  shows the difference between iterates decreases as the iterates approach the
  true solution, while \cref{lmth2} implies a contraction in the VI sense.  Full
  proofs are provided in supplementary material;  \cref{eq:lm1} and
  \cref{eq:lm2} are supported using equations (\ref{eq:optvi}, \ref{eq:gvi},
  \ref{eq:kpvi2}) and standard techniques, while \cref{eq:lm3} is
proven from \cref{eq:lm2}.  \cref{lmth2} is supported by the relationship in \cref{eq:zmtz}.

%
%

\begin{lemma}\label{lm3}
The optimal solution $z\opt = (u\opt; v\opt; \gl\opt)$ and sequence $z\itk = (u\itk; v\itk; \gl\itk)$ of generalized ADMM satisfy
\begin{align}
 (B\resp{v\subk})\trans \resp{\gl\subk} &\geq  0, \label{eq:lm1}\\
 \reso{z\subkp} H\itk \resp{z\subk} &\geq  0, \label{eq:lm2}\\
\| \resp{z\subk}\|^2_{H\itk} &\leq  \| \reso{z\subk} \|^2_{H\itk} - \| \reso{z\subkp} \|^2_{H\itk}.  \label{eq:lm3} 
\end{align}
\end{lemma}

\begin{lemma}\label{lmth2}
The sequence $\tz\itk = (u\itk; v\itk; \hat\gl\itk)$ and $z\itk = (u\itk; v\itk; \gl\itk)\trans$ from generalized ADMM satisfy, $\forall z,$
{\small%
\begin{equation}
(\tz\kp - z)\trans H\itk \resp{z\subk} \geq \frac{1}{2} (\| z\kp-z\|^2_{H\itk} - \| z\itk - z\|^2_{H\itk}).
\end{equation}
}%
\end{lemma}

\subsection{Convergence criteria}
We provide a convergence analysis of ADMM with an adaptive diagonal penalty
matrix by showing  
\begin{inparaenum}[(i)]
\item the norm of the residuals converges to zero;
\item the method attains a worst-case ergodic $O(1/k)$ convergence rate in the VI sense. 
\end{inparaenum}
The key idea of the proof is to bound the adaptivity of $T^k$ so that ADMM is stable enough to converge, which is presented as the following assumption. 

\begin{assumption}\label{as1} The adaptivity of the diagonal penalty matrix $T\itk = \text{diag}(\tau\subi\itk, \ldots, \tau_{p}\itk) $ is bounded by
\begin{equation}
\begin{split}
& \sum_{k=1}^{\infty} (\eta\itk)^2 < \infty, \text{ where }  (\eta\itk)^2 = \max_{ i \in \{1,\ldots,p\} }\{ (\eta\subi\itk)^2 \}, 
\\
& \quad (\eta\subi\itk)^2 = \max\{ \tau\subi\itk/\tau\subi\km-1,  \tau\subi\km/\tau\subi\itk-1 \}.
\end{split}
\end{equation}
\end{assumption}
We can apply \cref{as1} to verify that 
\begin{equation}\label{eq:bdadp}
\frac{1}{1+(\eta\itk)^2} \leq \frac{\tau\subi\itk}{\tau\subi\km} \leq 1+(\eta\itk)^2.
\end{equation}
which is needed to prove \cref{lmth1}.

\begin{lemma}\label{lmth1}
 Suppose \cref{as1} holds. Then $z = (u;\, v;\, \gl)$ and $z' = (u';\, v';\, \gl')$ satisfy, $\forall z, z'$
\begin{equation}
\| z-z' \|^2_{H\itk} \leq  (1+(\eta\itk)^2) \|  z - z' \|^2_{H\km}.
\end{equation}
\end{lemma}


Now we are ready to prove the convergence of generalized ADMM with adaptive penalty under \cref{as1}. We prove the following quantity, which is a norm of the residuals, converges to zero.
\begin{equation}
\begin{split}
\| \resp{z\subk}\|^2_{H\itk} = & \| B\resp{v\subk} \|^2_{T\itk} + \| \resp{\gl\subk}\|^2_{(T\itk)\inv} \\
= & \| (A\trans T\itk)\pinv d\itk \|^2_{T\itk} + \| r\itk \|^2_{T\itk},
\end{split}
\end{equation}
where $A\pinv$ denotes generalized inverse of a matrix A. Note that $\| \resp{z\subk}\|^2_{H\itk}$ converges to zero only if $\|
r\itk \|$ and $\| d\itk \|$ converge to zero, provided $A$ and $T\itk$ are
bounded.

\begin{theorem}\label{thm:acadmm1}
 Suppose \cref{as1} holds.
 Then the iterates 
 $z\itk = (u\itk; v\itk; \gl\itk)$
of generalized ADMM satisfy
\begin{equation}
\lim_{k\rightarrow \infty} \| \resp{z\subk}\|^2_{H\itk} = 0.
\end{equation}
\end{theorem}

\begin{proof}
Let $z=z\itk, z' = z\opt$ in \cref{lmth1} to achieve
\begin{equation}
\| \reso{z\subk} \|^2_{H\itk} \leq  (1+(\eta\itk)^2) \| \reso{z\subk} \|^2_{H\km}.\label{eq:thmt1}
\end{equation}
Combine \eqref{eq:thmt1} with \cref{lm3} \eqref{eq:lm3} to get
\begin{equation}
\| \resp{z\subk} \|^2_{H\itk}  \leq (1+(\eta\itk)^2) \| \reso{z\subk} \|^2_{H\km} - \|\reso{z\subkp}\|^2_{H\itk}. \label{eq:thmtmp}
\end{equation}
Accumulate \eqref{eq:thmtmp} for $k =1$ to $l$,
\begin{equation}
\begin{split}
& \sum_{k=1}^{l} \prod_{t=k+1}^{l} (1+(\eta^t)^2)\|\resp{z\subk} \|^2_{H\itk} \leq \\
&\quad \prod_{t=1}^{l} (1+(\eta^t)^2)\| \reso{z_{1}}\|^2_{H^{0}} - \|\reso{z_{l+1}}\|^2_{H^{l}}.
\end{split}
\end{equation}
Then we have
\begin{equation}
\sum_{k=1}^{l}\| \resp{z\subk} \|^2_{H\itk}  \leq \prod_{t=1}^{l} (1+(\eta^t)^2)\| \reso{z_{1}}\|^2_{H^{0}}.
\end{equation}
When $l\rightarrow \infty$, \cref{as1} suggests $\prod_{t=1}^{\infty} (1+(\eta^t)^2) < \infty$, which means $\sum_{k=1}^{\infty}\| \resp{z\subk} \|^2_{H\itk} < \infty$. Hence
$
\lim_{k\rightarrow \infty} \| \resp{z\subk}\|^2_{H\itk} = 0.
$
\end{proof}

We further exploit \cref{as1} and \cref{lmth1} to prove \cref{lmth2_2}, and combine VI \eqref{eq:kpvi2}, \cref{lmth2}, and \cref{lmth2_2} to prove the $O(1/k)$ convergence rate in \cref{thm:acadmm2}. 
\begin{lemma}\label{lmth2_2}
 Suppose \cref{as1} holds. Then $z = (u; v; \gl) \in \bbR^{m+n+p}$ and the iterates $z\itk = (u\itk; v\itk; \gl\itk)$ of generalized ADMM satisfy, $\forall z$
\begin{equation}
\begin{split}
& \sum_{k=1}^{l} (\| z-z\itk \|^2_{H\itk} - \| z-z\itk \|^2_{H\km}) \leq \\
 & \qquad C_{\eta}^{\gS} C_{\eta}^{\Pi} ( \|  z-z\opt \|^2_{H^{0}} +  \| \reso{z_{1}} \|^2_{H^{0}}) < \infty,
 \end{split}
\end{equation}
 where $C_{\eta}^{\gS} = \sum_{k=1}^{\infty} (\eta\itk)^2 $, $C_{\eta}^{\Pi} = \prod_{t=1}^{\infty} (1+(\eta^{t})^2) $.
\end{lemma}

\begin{theorem}\label{thm:acadmm2}
 Suppose \cref{as1} holds.
Consider the sequence $\tz^{k} = (u^{k}; v^{k}; \hat\gl^{k})$ of generalized ADMM  and define
$
\bar{z}^{l} = \frac{1}{l} \sum_{k=1}^{l} \tz^{k}. \,
$
Then sequence $\bar{z}^{l}$ satisfies the convergence bound
{ %
\begin{multline}
\phi(y)-\phi(\bar y^{l}) + (z-\bar z^{l})\trans F(\bar z^{l}) \geq  -\frac{1}{2\, l}  \large(\| z - z^0 \|^2_{H^0}  \\
  +  C_{\eta}^{\gS} C_{\eta}^{\Pi}  \|  z-z\opt \|^2_{H^{0}} + C_{\eta}^{\gS} C_{\eta}^{\Pi} \| \reso{z_{1}} \|^2_{H^{0}}\large).
 \label{eq:thm2}
\end{multline}
}%
\end{theorem}

\begin{proof}
We can verify with simple algebra that 
\begin{equation}
(z-z')\trans F(z) = (z-z')\trans F(z'). \label{eq:th2t1}
\end{equation}
Apply \eqref{eq:th2t1} with $z'=\tz\kp$, and combine VI \eqref{eq:kpvi2} and \cref{lmth2} to get
\begin{align}
& \phi(y)-\phi(y\kp) + (z-\tz\kp)\trans F(z)  \\
 = & \phi(y)-\phi(y\kp) + (z-\tz\kp)\trans F(\tz\kp)  \\
 \geq & (\tz\kp-z)\trans H\itk \resp{z\subk} \\
 \geq & \frac{1}{2} (\| z\kp-z\|^2_{H\itk} - \| z\itk - z\|^2_{H\itk}).
\end{align}
Summing for $k=0$ to $l-1$ gives us
\begin{align}
\begin{split}
& \sum\nolimits_{k=1}^{l} \phi(y)-\phi(y\itk) + (z-\tz\itk)\trans F(z)  \\
 \geq & \frac{1}{2} \sum\nolimits_{k=1}^{l} \large(\| z - z\itk\|^2_{H\km} - \| z - z\km \|^2_{H\km}\large). \label{eq:th2t2}
 \end{split}
\end{align}
Since $\phi(y)$ is convex, the left hand side of \eqref{eq:th2t2} satisfies,
\begin{align}
LHS = & \, l\, \phi(y) - \sum_{k=1}^{l} \phi(y\itk) + (l \, z- \sum_{k=1}^{l}\tz\itk)\trans F(z)  \nonumber \\
 \leq & \, l\, \phi(y) - l\, \phi(\bar y^l) + (l \, z- l\, \bar z^l)\trans F(z).  \label{eq:th2t3}
\end{align}
Applying \cref{lmth2_2}, we see the right hand side satisfies,
\begin{align}
 \begin{split}
  RHS = & \frac{1}{2} \sum_{k=1}^{l} (\| z - z\itk\|^2_{H\itk} - \| z - z\km \|^2_{H\km}) + \\
  & \quad \frac{1}{2} \sum_{k=1}^{l} (\| z - z\itk\|^2_{H\km} - \| z - z\itk \|^2_{H\itk})
  \end{split}\\
  \begin{split}
  \geq & \frac{1}{2}  (\| z - z^l\|^2_{H^l} - \| z - z^0 \|^2_{H^0}) + \\
  & \quad -\frac{1}{2} C_{\eta}^{\gS} C_{\eta}^{\Pi} ( \|  z-z\opt \|^2_{H^{0}} +  \| \reso{z_{1}} \|^2_{H^{0}})
  \end{split}\\
  \begin{split}
    \geq & -\frac{1}{2}  (\| z - z^0 \|^2_{H^0} + C_{\eta}^{\gS} C_{\eta}^{\Pi}  \|  z-z\opt \|^2_{H^{0}} + \\
     &\qquad C_{\eta}^{\gS} C_{\eta}^{\Pi} \| \reso{z_{1}} \|^2_{H^{0}}).
     \end{split}\label{eq:th2t4}
\end{align}
Combining inequalities \eqref{eq:th2t2}, \eqref{eq:th2t3} and \eqref{eq:th2t4}, and letting $z' = \bar z\itk$ in \eqref{eq:th2t1} yields the $O(1/k)$ convergence rate in \eqref{eq:thm2}
\end{proof}

\section{Adaptive Consensus ADMM (ACADMM)} \label{sec:adapt}
To address the issue of how to automatically tune parameters on each node for optimal performance,
we propose \emph{adaptive consensus ADMM} (ACADMM), 
which sets worker-specific penalty parameters by exploiting curvature
information. We derive our method from the dual interpretation of ADMM --
\emph{Douglas-Rachford splitting} (DRS) -- using a diagonal penalty matrix. We
then derive the spectral stepsizes for consensus problems by assuming the
curvatures of the objectives are diagonal matrices with diverse parameters on
different nodes.  At last, we discuss the practical computation of the
spectral stepsizes from consensus ADMM iterates and apply our theory in
\cref{sec:theory} to guarantee convergence. 

\subsection{Dual interpretation of generalized ADMM}
\label{sec:admmdrs}
The dual form of problem \eqref{eq:prob} can be written 
\begin{equation}
\min_{\gl\in \bbR^p}  \underbrace{f^*(A^{T}\gl) - \inprod{\gl, b}}_{ \hat{f}(\gl)} + \underbrace{g^*(B^{T}\gl)}_{\hat{g}(\gl)},
\label{eq:dprob}
\end{equation}
where $\gl$ denotes the dual variable, while $f^*, g^*$ denote the Fenchel conjugate of $f,g$ \cite{Rockafellar}.
It is known that ADMM steps for the primal problem \eqref{eq:prob} are equivalent to performing \emph{Douglas-Rachford splitting} (DRS) on the dual problem \eqref{eq:dprob} \cite{eckstein1992douglas,xu2016adaptive}. In particular, the generalized ADMM iterates satisfy the DRS update formulas
\begin{align}
0 & \in (T\itk)\inv (\hat\gl\kp - \gl\itk) + \partial \hat f(\hat \gl \kp) + \partial \hat g(\gl\itk) \label{eq:gdrs1}\\
0 & \in (T\itk)\inv (\gl\kp - \gl\itk) + \partial \hat f(\hat \gl \kp) + \partial \hat g(\gl\kp), \label{eq:gdrs2}
\end{align}
where $\hat \gl$ denotes the intermediate variable defined in \cref{sec:pre}.
We prove the equivalence of generalized ADMM and DRS in the
supplementary material.

\subsection{Generalized spectral stepsize rule}
\label{sec:proposition}
\citet{xu2016adaptive} first derived spectral penalty parameters for ADMM using the DRS.
\cref{prop:gen} in \cite{xu2016adaptive} proved that the minimum residual of DRS can be obtained by setting the scalar penalty to $\tau\itk =1/\sqrt{\alpha\, \beta}$, where we assume the subgradients are locally linear as
\begin{equation}%
\partial \hat f(\hat \gl) = \alpha \, \hat \gl + \Psi ~~~~~\text{and}~~~~~\partial \hat g(\gl) = \beta \, \gl + \Phi, \label{eq:loclin}
\end{equation} %
$\ga,\gb \in \bbR$ represent scalar curvatures, and $\Psi, \Phi \subset \bbR^p$.

We now present generalized spectral stepsize rules that can accomodate consensus problems.


\begin{proposition}[Generalized spectral DRS] \label{prop:gen}
Suppose the generalized DRS steps (\ref{eq:gdrs1}, \ref{eq:gdrs2}) are used,
and assume the subgradients are locally linear,
\begin{equation}%
\partial \hat f(\hat \gl) = M_{\ga} \, \hat \gl + \Psi ~~~~~\text{and}~~~~~\partial \hat g(\gl) = M_{\gb} \, \gl + \Phi . \label{eq:gloclin}
\end{equation} %
for matrices $M_{\ga} = \text{diag}(\ga_1 I_d,\ldots, \ga_N I_d)$ and $M_{\gb} = \text{diag}(\gb_1 I_d,\ldots, \gb_N I_d)$, and some $\Psi, \Phi \subset \bbR^p$. Then the minimal residual of $\hat{f}(\gl\kp) + \hat{g}(\gl\kp)$ is obtained by setting $\tau\subi\itk =1/\sqrt{\alpha\subi \, \beta\subi},\, \forall i=1,\ldots,N$.
\label{pp:gspectral}
\end{proposition}

\begin{proof}
Substituting subgradients $\partial \hat f(\hat \gl), \partial \hat g(\gl)$ into the generalized DRS steps (\ref{eq:gdrs1}, \ref{eq:gdrs2}), and using our linear assumption \eqref{eq:gloclin} yields 
{\small%
\begin{align*}
0  &\in (T\itk)\inv (\hat\gl\kp - \gl\itk) + (M_{\ga} \, \hat \gl \kp + \Psi) + (M_{\gb} \, \gl \itk + \Phi) \\
0  &\in (T\itk)\inv (\gl\kp - \gl\itk) + (M_{\ga} \, \hat \gl \kp + \Psi) + (M_{\gb} \, \gl \kp + \Phi).
\end{align*}
}%
Since $T\itk, M_{\ga}, M_{\gb}$ are diagonal matrices, we can split the equations into independent blocks, $\forall i = 1,\ldots, N,$
{\small%
\begin{align*}
0  &\in (\hat\gl\subi\kp - \gl\subi\itk)/\tau\subi\itk + (\ga\subi \, \hat \gl \kp + \Psi\subi) + (\gb\subi \, \gl \itk + \Phi\subi) \\
0  &\in  (\gl\subi\kp - \gl\subi\itk)/\tau\subi\itk + (\ga\subi \, \hat \gl \kp + \Psi\subi) + (\gb\subi \, \gl \kp + \Phi\subi).
\end{align*}
}%
Applying \cref{prop:gen} in \cite{xu2016adaptive}
to each block, $\tau\subi\itk =1/\sqrt{\alpha\subi \, \beta\subi}$ minimizes the block residual represented by $r_{DR,i}\kp = \| (\ga\subi+\gb\subi)\gl\kp + (a\subi+b\subi)\|$, where $a\subi\in \Psi\subi, b\subi\in\Phi\subi$. Hence the residual norm at step $k+1,$ which is $ \| (M_{\ga}+M_{\gb})\gl\kp + (a+b)\| = \sqrt{\sum_{i=1}^N (r_{DR,i}\kp)^2}$ is minimized by setting $\tau\subi\itk =1/\sqrt{\alpha\subi \, \beta\subi},\, \forall i=1,\ldots,N$.
\end{proof}

\subsection{Stepsize estimation for consensus problems}
Thanks to the equivalence of ADMM and DRS, \cref{pp:gspectral} can also be used to guide the selection of the ``optimal'' penalty parameter. We now show that the generalized spectral stepsizes can be estimated from the ADMM iterates for the primal consensus problem \eqref{eq:cprob}, without explicitly supplying the dual functions. 

The subgradients of dual functions $\partial \hat f, \, \partial \hat g$ can be computed from the ADMM iterates using the identities derived from (\ref{eq:gadmm1}, \ref{eq:gadmm2}), 
\begin{equation} \label{eq:conjgrad}
 Au\kp-b \in  \partial \hat f( \hat{\gl}\kp) \ \text{ and } \ Bv\kp \in \partial \hat{g}(\gl\kp).
 \end{equation}
For the consensus problem we have $A=I_{dN}$, $B = -(I_d;\ldots;I_d) $,  and $b=0,$ and so 
\begin{align} 
 (u_1\kp; \, \ldots; \, u_N\kp) & \in  \partial \hat f( \hat{\gl}\kp) \\
  -(\underbrace{v\kp; \, \ldots; \, v\kp}_{N \text{ duplicates of }v\kp}) & \in \partial \hat{g}(\gl\kp).
 \end{align}
If we approximate the behavior of these sub-gradients using the linear approximation \eqref{eq:gloclin}, and break the sub-gradients into blocks (one for each worker node), we get (omitting iteration index $k$ for clarity)
\begin{equation}%
u\subi = \ga\subi \, \hat \gl \subi + a\subi \text{ and } -v = \gb\subi \, \gl \subi + b\subi, \,\,\forall i \label{eq:spectralll}
\end{equation}%
where $\ga_i$ and $\gb_i$ represent the {\em curvature} of local functions $\hat f_i$ and $\hat g_i$ on the $i$th node.

We select stepsizes with a two step procedure, which follows the spectral stepsize literature. First, we estimate the local curvature parameters, $\ga_i$ and $\gb_i,$ by finding least-squares solutions to \eqref{eq:spectralll}.  Second, we plug these curvature estimates into the formula $\tau\subi\itk =1/\sqrt{\alpha\subi \, \beta\subi}.$ This formula produces the optimal stepsize when $\hat f$ and $\hat g$ are well approximated by a linear function, as shown in Proposition~\ref{prop:gen}.

For notational convenience, we work with the quantities
$
\hat\ga\itk\subi = 1/\ga\subi, \ \hat\gb\itk\subi = 1/\gb\subi,
$
which are estimated on each node using the current iterates $u\subi\itk, v\itk, \gl\subi\itk, \hat\gl\subi\itk$ and also an older iterate $u\subi\ito, v\ito, \gl\subi\ito, \hat\gl\subi\ito, k_0 < k$. Defining
$
\gD u\subi\itk = u\subi\itk - u\subi\ito, \ \gD \hat\gl\subi\itk = \hat\gl\subi\itk - \hat\gl\subi\ito
$
and following the literature for Barzilai-Borwein/spectral stepsize estimation, there are two least squares estimators that can be obtained from~\eqref{eq:spectralll}:
\begin{equation}
\hga\itk_{\mbox{\scriptsize SD}, i} = \frac{\inprod{\gD \hat{\gl}\itk\subi, \gD \hat{\gl}\itk\subi}}{\inprod{\gD u\itk\subi, \gD \hat{\gl}\itk\subi}}
\,\ \text{and} \,\
\hga\itk_{\mbox{\scriptsize MG}, i} = \frac{\inprod{\gD u\itk\subi, \gD \hat{\gl}\itk}}{\inprod{\gD u\itk\subi, \gD u\itk\subi}} \label{eq:alpha2}
\end{equation}
where SD stands for {\em steepest descent}, and MG stands for {\em minimum gradient}.
\citep{zhou2006gradient} recommend using a hybrid of these two estimators, and choosing 
\begin{eqnarray}\label{eq:alpha}
\hat{\ga}\subi\itk =
\begin{cases}
\hat{\ga}\itk_{\mbox{\scriptsize MG},i}& \ \text{if}\ 2 \,\hat{\ga}\itk_{\mbox{\scriptsize MG}, i} > \hat{\ga}\itk_{\mbox{\scriptsize SD}, i} \\
\hat{\ga}\itk_{\mbox{\scriptsize SD}, i} - \hat{\ga}\itk_{\mbox{\scriptsize MG}, i} /2  &~~\text{otherwise.}
\end{cases}
\end{eqnarray}
It was observed that this choice worked well for non-distributed ADMM in \cite{xu2016adaptive}.
We can similarly estimate $\hat\gb\subi\itk$ from
$\gD v\itk = - v\itk + v\ito$ and $\gD \gl\subi\itk = \gl\subi\itk - \gl\subi\ito
$. 

ACADMM estimates the curvatures in the original $d$-dimensional feature space, and avoids
estimating the curvature in the higher $Nd$-dimensional feature space (which grows
with the number of nodes $N$ in AADMM \cite{xu2016adaptive}), which is especially useful for heterogeneous data with different distributions allocated to different
nodes. The overhead of our adaptive scheme is only a few inner products, and the computation is naturally distributed on
different workers.

\begin{algorithm}[t]
\caption{Adaptive consensus ADMM (ACADMM)}
\label{alg}
\begin{algorithmic}[1]
\REQUIRE  initialize $v^0$, $\gl\subi^0$, $\tau\subi^0$, $k_0 \! = \! 0$,
\WHILE{not converge by \eqref{eq:stop} \textbf{and} $k <\text{maxiter}$}
\STATE Locally update $u\subi\itk$ on each node by 
\eqref{eq:cadmms}
\STATE Globally update $v\itk$ on central server by 
\eqref{eq:cadmm2}
\STATE Locally update dual variable $\gl\subi\itk$  on each node by 
\eqref{eq:cadmme}
\IF{$\text{mod}(k, T_{f}) = 1$}
\STATE Locally update $\hgl\subi\itk = \gl\subi\km +\tau\subi\itk (v\km - u\subi\itk)$
\STATE Locally compute spectral stepsizes $\hga\subi\itk,\hgb\subi\itk$ 
\STATE Locally estimate correlations $\ga\itk_{\mbox{\scriptsize cor}, i} \, , \, \gb\itk_{\mbox{\scriptsize cor}, i}$ 
\STATE Locally update $\tau\subi\kp$ using \eqref{eq:final}
\STATE  $k_0 \gets k$
\ELSE
\STATE $\tau\subi\kp \gets \tau\subi\itk$
\ENDIF
\STATE $k \gets k+1$
\ENDWHILE
\end{algorithmic}
\end{algorithm} 
\begin{table*}[t]
\vspace{-2mm}
\centering
\caption{\small Iterations (and runtime in seconds);128 cores are used; absence of convergence after $n$ iterations is indicated as $n+$.
\vspace{-2mm}
}
\setlength{\tabcolsep}{3pt}
\small
\begin{threeparttable}
\begin{tabular}{|c|c|c||c|c|c|c|>{\bfseries}c|}
\hline
Application & Dataset & \tabincell{c}{\#samples $\times$ \\ \#features
\textsuperscript{1}
} & \tabincell{c}{CADMM \\ \cite{boyd2011admm}} & \tabincell{c}{RB-ADMM \\ \cite{he2000alternating}} & \tabincell{c}{AADMM \\ \cite{xu2016adaptive}} & \tabincell{c}{CRB-ADMM \\ \cite{song2015fast}} &  \tabincell{c}{Proposed \\ ACADMM} \\
\hline\hline
\multirow{7}{*}{\tabincell{c}{Elastic net\\regression}}
&Synthetic1 & 64000 $\times$ 100 &   1000+(1.27e4) &	94(1.22e3)	& \textbf{43(563)} & 106(1.36e3) &	48(623) \\
&Synthetic2 & 64000 $\times$ 100 & 1000+(1.27e4)	 & 130(1.69e3) & 341(4.38e3)	& 140(1.79e3)	& 57(738) \\
& MNIST & 60000 $\times$ 784 & 100+(1.49e4) & 88(1.29e3) & 40(5.99e3) & 87(1.27e4) & 14(2.18e3) \\
& CIFAR10
\textsuperscript{2}
& 10000 $\times$ 3072 & 100+(1.04e3) & 100+(1.06e3) & 100+(1.05e3) & 100+(1.05e3) & 35(376) \\
& News20 & 19996 $\times$ 1355191 & 100+(4.61e3) & 100+(4.60e3) & 100+(5.17e3) & 100+(4.60e3) & 78(3.54e3) \\
& RCV1 & 20242 $\times$ 47236  & 33(1.06e3) & 31(1.00e3) & 20(666) & 31(1.00e3) & 8(284) \\
& Realsim & 72309 $\times$ 20958 & 32(5.91e3) & 30(5.59e3) & 14(2.70e3) & 30(5.57e3) & 9(1.80e3) \\
\hline
\multirow{7}{*}{\tabincell{c}{Sparse \\ logistic \\ regression}}
&Synthetic1 & 64000 $\times$ 100 &  138(137) &	78(114)  & 80(101) &	48(51.9)&	24(29.9) \\
&Synthetic2 & 64000 $\times$ 100 &  317(314) &	247(356) &	1000+(1.25e3) &	1000+(1.00e3)&	114(119)\\
& MNIST & 60000 $\times$ 784 & 325(444) & 212(387) & 325(516) & 203(286) & 149(218) \\
& CIFAR10 & 10000 $\times$ 3072 & 310(700) & 152(402) & 310(727) & 149(368) & 44(118) \\
& News20 & 19996 $\times$ 1355191 & 316(4.96e3) & 211(3.84e3) & 316(6.36e3) & 207(3.73e3) & 137(2.71e3) \\
& RCV1 & 20242 $\times$ 47236  & 155(115) & 155(116) & 155(137) & 155(115) & 150(114) \\
& Realsim & 72309 $\times$ 20958 & 184(77) & 184(77) & 184(85) & 183(77) & 159(68) \\
\hline
\multirow{7}{*}{\tabincell{c}{Support \\ Vector \\ Machine}} 
&Synthetic1 & 64000 $\times$ 100 &  33(35.0) &	33(49.8) &	\textbf{19(27)} &	26(28.4)	& 21(25.3) \\
&Synthetic2 & 64000 $\times$ 100 &  283(276) &	69(112) &	1000+(1.59e3) &	81(97.4) &	25(39.0)\\
& MNIST & 60000 $\times$ 784 & 1000+(930) & 172(287) & 73(127) & 285(340) & 41(88.0) \\
& CIFAR10 & 10000 $\times$ 3072 & 1000+(774) & 227(253) & 231(249) & 1000+(1.00e3) & 62(60.2) \\
& News20 & 19996 $\times$ 1355191 & 259(2.63e3) & 262(2.74e3) & 259(3.83e3) & 267(2.78e3) & 217(2.37e3) \\
& RCV1 & 20242 $\times$ 47236  & 47(21.7) & 47(21.6) & 47(31.1) & 40(19.0) & 27(15.4) \\
& Realsim & 72309 $\times$ 20958 & 1000+(76.8) & 1000+(77.6) & 442(74.4) & 1000+(79.3) & 347(41.6) \\
\hline
\multirow{1}{*}{\tabincell{c}{SDP}}
& Ham-9-5-6 & 512 $\times$ 53760 & 100+(2.01e3) & 100+(2.14e3) & 35(860) & 100+(2.14e3) & 30(703) \\
\hline
\end{tabular}%
\label{tab:exp}%
\begin{tablenotes}
    \item \textsuperscript{1} \#vertices $\times$ \#edges for SDP; \qquad \textsuperscript{2}We only use the first training batch of CIFAR10.
\end{tablenotes}
\end{threeparttable}
\vspace{-4mm}
\end{table*}%

\subsection{Safeguarding and convergence}
Spectral stepsizes for gradient descent methods are equipped with
safeguarding strategies like backtracking line search to handle inaccurate
curvature estimation and to guarantee convergence.  To safeguard the proposed
spectral penalty parameters, we check whether our linear
subgradient assumption is reasonable before updating the stepsizes.  
We do this by testing that the correlations
\begin{equation}\label{eq:corr}
\small
\ga\itk_{\mbox{\scriptsize cor}, i} = \frac{\inprod{\gD u\subi\itk, \gD \hat{\gl}\subi\itk}}{ \| \gD u\subi \itk\| \, \| \gD \hat{\gl}\subi\itk\| } \ \, \text{and} \, \  
\gb\itk_{\mbox{\scriptsize cor}, i} = \frac{\inprod{\gD v\itk, \gD \gl\subi\itk}}{ \| \gD v \itk\| \, \| \gD \gl \subi\itk\| },
\end{equation}
are bounded away from zero by a fixed threshold.
We also bound changes in the penalty parameter by $(1+\nicefrac{C_{\mbox{\scriptsize cg}}}{k^2})$ according to \cref{as1}, which was shown in \cref{thm:acadmm1} and \cref{thm:acadmm2} to guarantee convergence.  The final safeguarded ACADMM rule is 
{\small%
\begin{equation} 
\begin{split}
\hat\tau\subi\kp = &
\begin{cases}
\sqrt{\hat{\ga}\subi\itk \hat{\gb}\subi\itk} &~~\text{if}~~ \ga_{\mbox{\scriptsize cor}, i}\itk > \gep^{\mbox{\scriptsize cor}}~~\text{and}~~\gb_{\mbox{\scriptsize cor}, i}\itk > \gep^{\mbox{\scriptsize cor}}\\
\hat{\ga}\subi\itk &~~\text{if}~~ \ga_{\mbox{\scriptsize cor}, i}\itk > \gep^{\mbox{\scriptsize cor}}~~\text{and}~~\gb_{\mbox{\scriptsize cor}, i}\itk \leq \gep^{\mbox{\scriptsize cor}}\\
\hat{\gb}\subi\itk &~~\text{if}~~ \ga_{\mbox{\scriptsize cor}, i}\itk \leq \gep^{\mbox{\scriptsize cor}}~~\text{and}~~\gb_{\mbox{\scriptsize cor}, i}\itk > \gep^{\mbox{\scriptsize cor}}\\
\tau\subi\itk &~~\text{otherwise},
\end{cases} 
\\
\tau\subi\kp = & \max\{\min\{\hat \tau\subi\kp, \, (1+\frac{C_{\mbox{\scriptsize cg}}}{k^2}) \tau\subi\itk\} \, , \,  \frac{\tau\subi\itk}{1+\nicefrac{C_{\mbox{\scriptsize cg}}}{k^2}} \}.
\end{split}
\label{eq:final}
\end{equation}
}%

\begin{figure*}[t]
\vspace{-2mm}
\centering 
        \begin{subfigure}[t]{0.32\textwidth}
                 \includegraphics[width=1.1\textwidth]{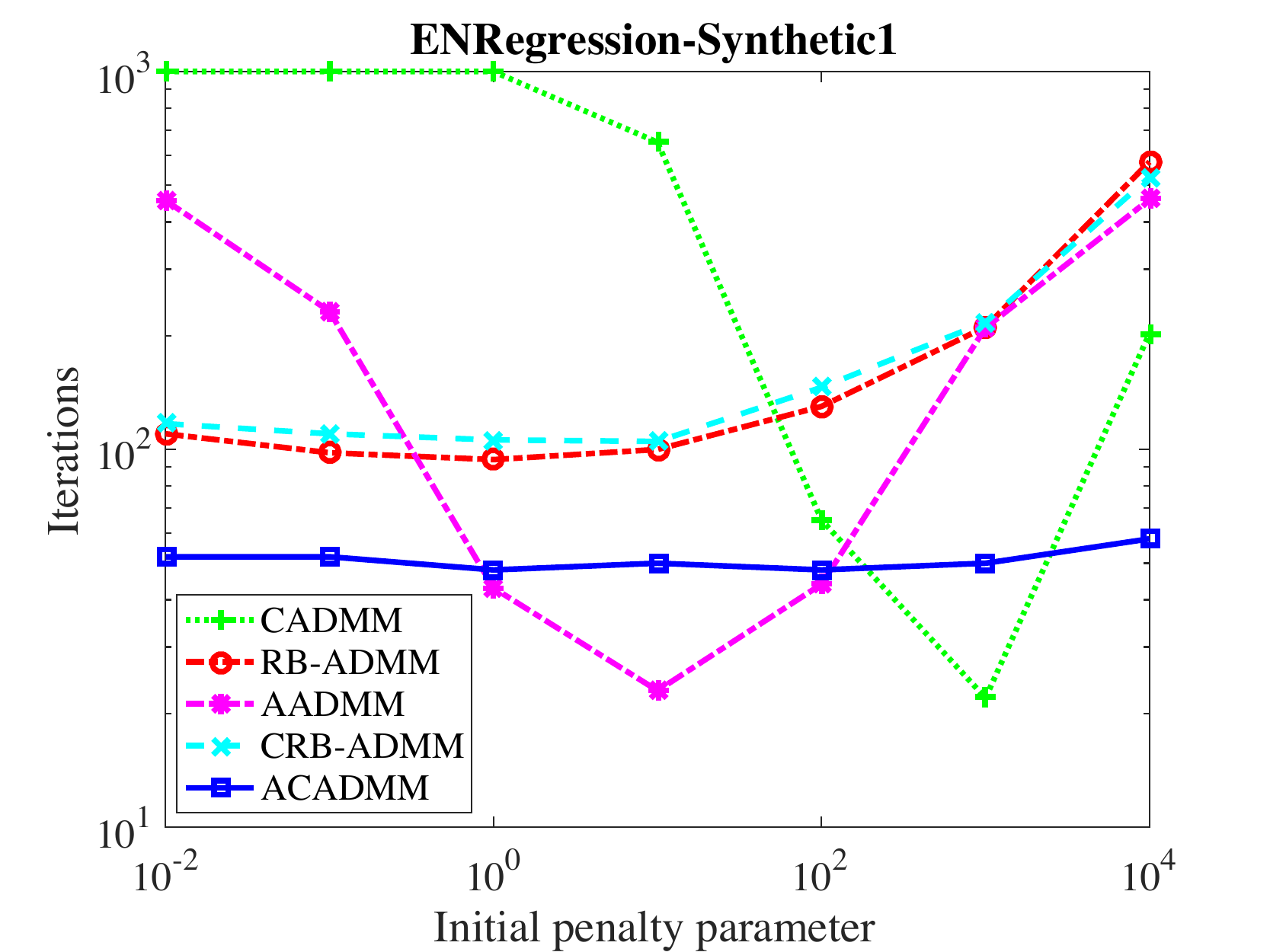}
         \end{subfigure}
         ~
         \begin{subfigure}[t]{0.32\textwidth}
                 \includegraphics[width=1.1\textwidth]{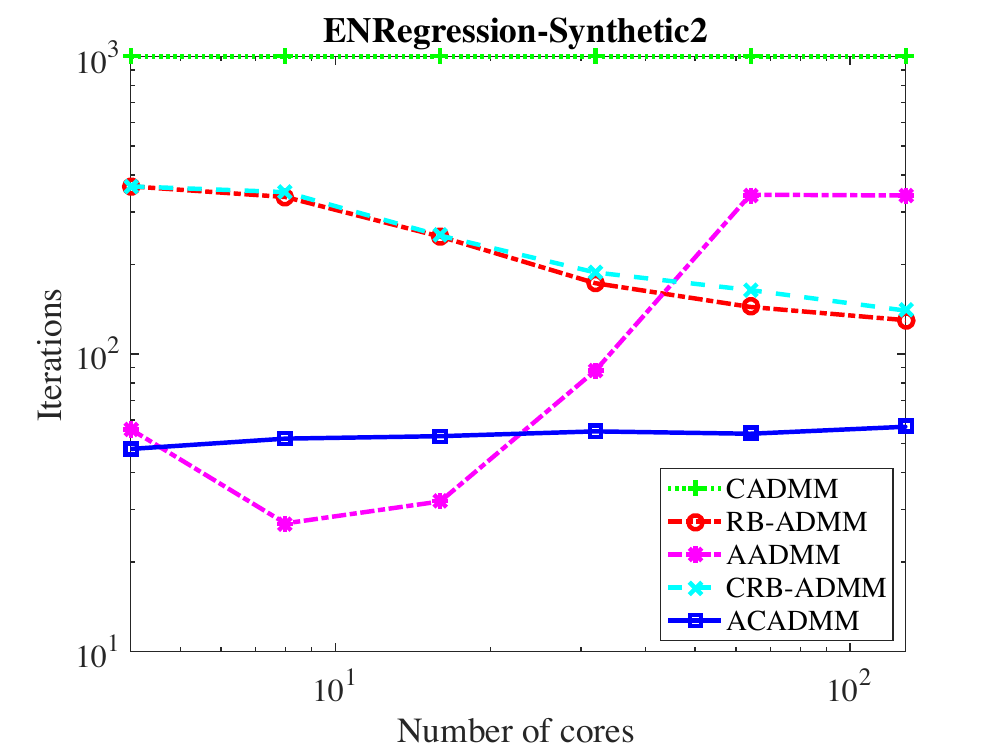}
         \end{subfigure}
         ~
         \begin{subfigure}[t]{0.32\textwidth}
                 \includegraphics[width=1.1\textwidth]{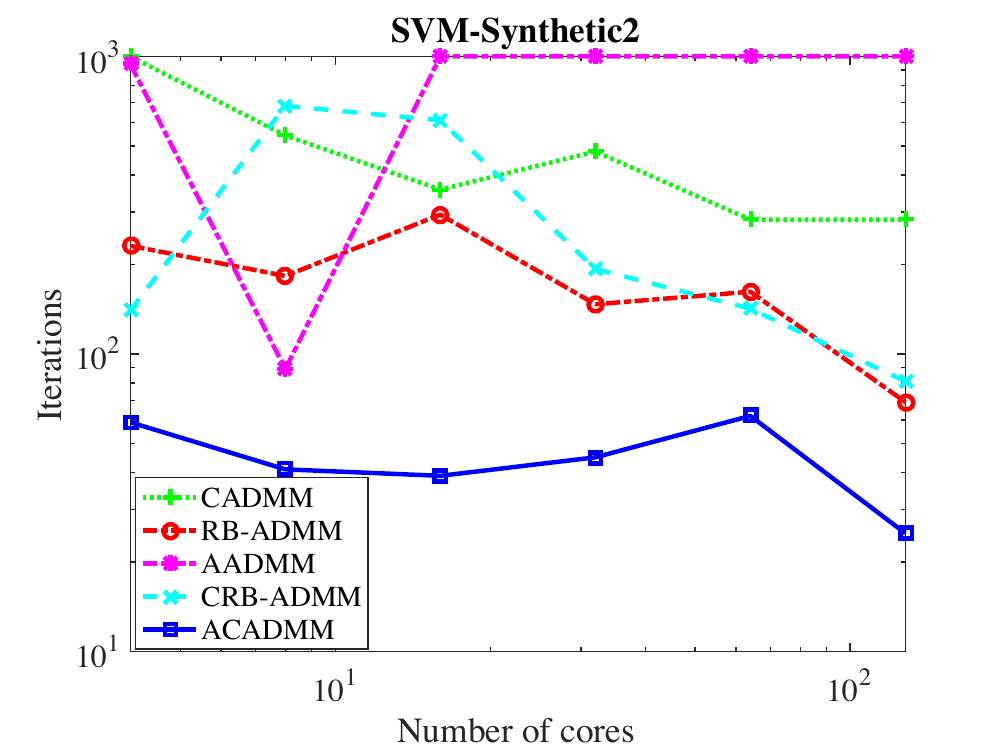}
         \end{subfigure}
         \\
         \begin{subfigure}[t]{0.32\textwidth}
                 \includegraphics[width=1.1\textwidth]{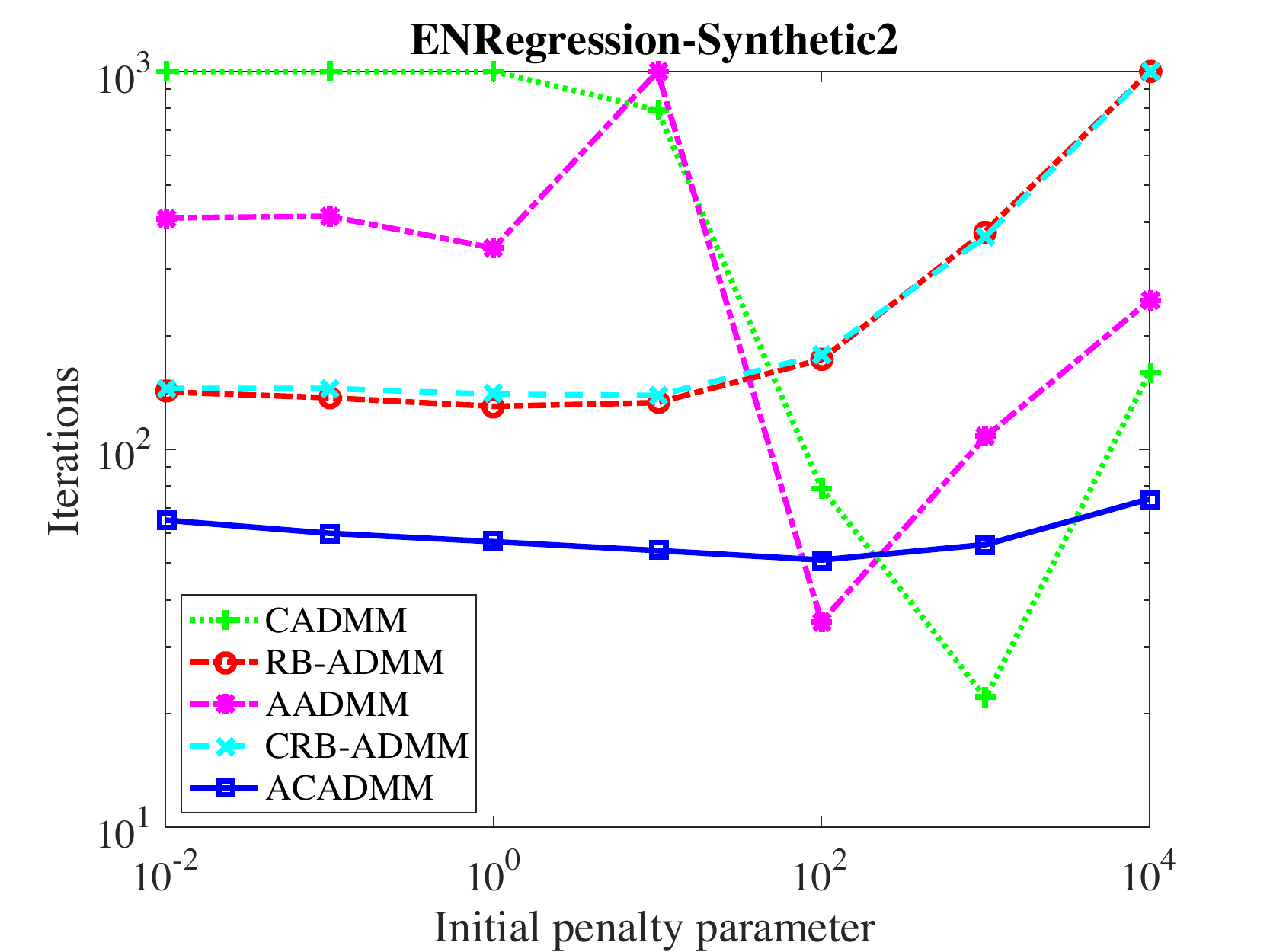}
                 \caption{\small\footnotesize Sensitivity of iteration count to initial penalty $\tau_0$. Synthetic problems of EN regression are studied with 128 cores.} 
                 \label{fig:tau}
         \end{subfigure}
         ~
         \begin{subfigure}[t]{0.32\textwidth}
                 \includegraphics[width=1.1\textwidth]{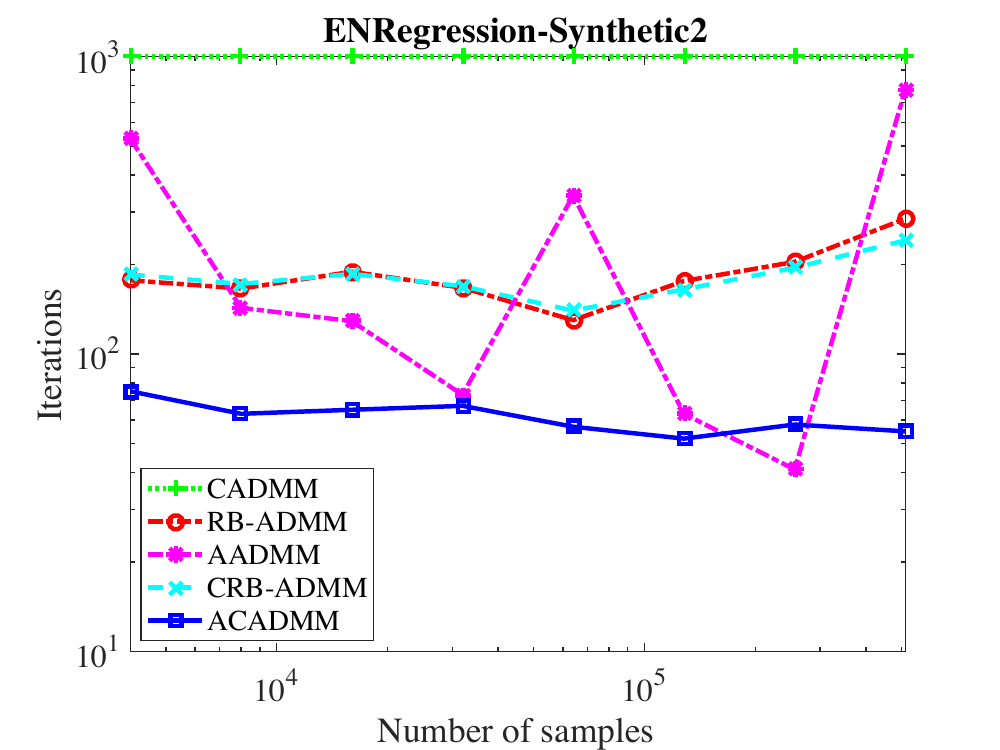}
                 \caption{\small\footnotesize Sensitivity of iteration count to number of cores (top) and number of samples (bottom).} 
					\label{fig:enr}
         \end{subfigure}
         ~
         \begin{subfigure}[t]{0.32\textwidth}
                 \includegraphics[width=1.1\textwidth]{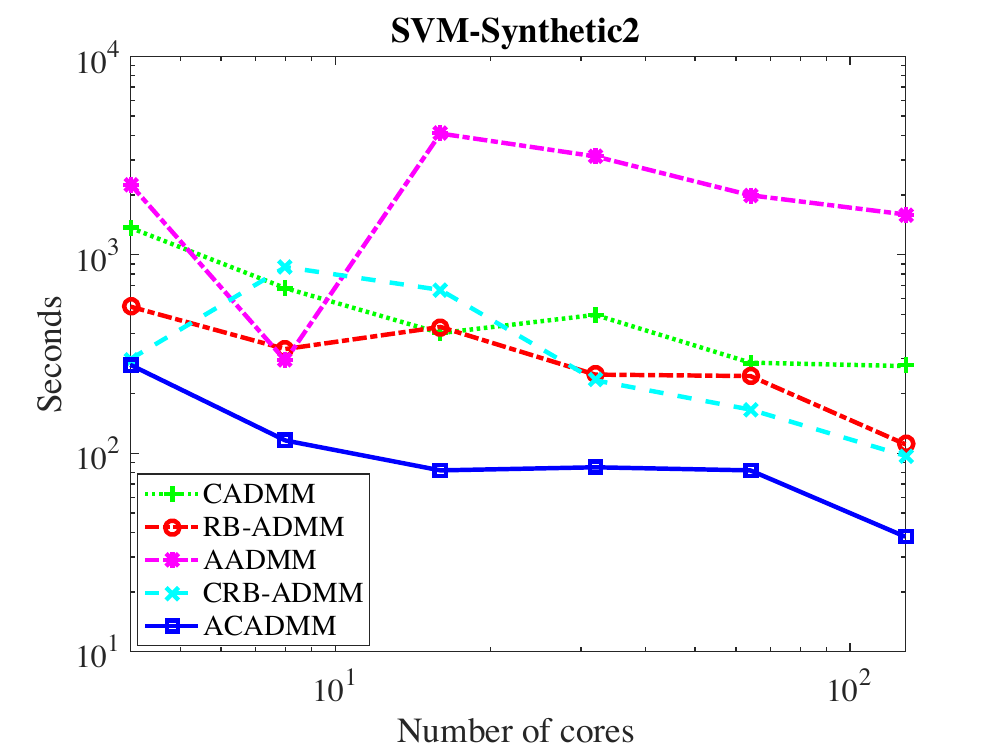}
                 \caption{\small\footnotesize Sensitivity of iteration count (top) and wall time (bottom) to number of cores.} 
					\label{fig:svm}
         \end{subfigure}
          \vspace{-3mm} 
         \caption{
         \small ACADMM is robust to the initial penalty $\tau$, number of cores $N$, and number of training samples.
           \vspace{-5mm}  }
\end{figure*}

The complete \textit{adaptive consensus ADMM} is shown in \cref{alg}.
We suggest updating the stepsize every $T_f = 2$ iterations, fixing the
safeguarding threshold $\gep^{\mbox{\scriptsize cor}}=0.2$, and choosing a
large convergence constant $C_{\mbox{\scriptsize cg}}=10^{10}$. 

\section{Experiments \& Applications}
We now study the performance of ACADMM on benchmark problems, and compare to other methods.
\subsection{Applications} \label{sec:app}
Our experiments use the following test problems that are commonly solved using consensus methods.

\textbf{Linear regression with elastic net regularizer.} 
We consider consensus formulations of the elastic net \cite{zou2005regularization} with
$f\subi$ and $g$ defined as,
\begin{equation}
f\subi(u\subi) = \frac{1}{2} \| D\subi u\subi -c\subi\|^2, \ g(v) = \rho_1 | v | + \frac{\rho_2}{2} \| v\|^2,
\end{equation}%
where $D\subi\in \bbR^{n\subi\times m} $ is the data matrix on node $i$, and $c\subi$ is a vector of measurements.

\textbf{Sparse logistic regression with $\ell_1$ regularizer} 
 can be written in the consensus form for distributed computing,
 {\small%
\begin{equation}
  f\subi(u\subi) =  \sum_{j=1}^{n_i} \log(1+\exp(-c_{i,j}D_{i,j}^Tu_i)),~ g(v) = \rho | v | 
\end{equation}
}%
where $D_{i,j} \in
\bbR^m$ is the $j$th sample, and $c_{i,j}\in \{-1,1\}$ is the corresponding
label. The minimization sub-step \eqref{eq:cadmms} in this case is solved
 by  L-BFGS \cite{liu1989limited}.

\textbf{Support Vector Machines (SVMs)}  minimize the distributed objective function \cite{goldstein2016unwrapping}
{\small%
\begin{equation}
 f\subi(u\subi) =  C \sum_{j=1}^{n_i} \max\{1-c_{i,j}D_{i,j}^Tu\subi, 0\},~ g(v) = \frac{1}{2} \| v\|_2^2
\end{equation}
}%
where $D_{i,j} \in \bbR^m$ is the $j$th sample on the $i$th node, and
$c_{i,j}\in \{-1,1\}$ is its label. The minimization
\eqref{eq:cadmms} is solved by dual coordinate ascent
\cite{chang2011libsvm}. 

\textbf{Semidefinite programming (SDP)} can be distributed as,
{\small%
\begin{equation}
 f\subi(U\subi) =  \iota\{\mcD\subi(U\subi) = c\subi\},~ g(v) = \inprod{F, V} + \iota\{V \succeq 0\}
\end{equation}
}%
where $\iota\{S\}$ is a characteristic function that is 0 if condition $S$ is satisfied and infinity otherwise. $V\!\! \succeq \!\! 0$ indicates that $V$ is positive semidefinite. $V, \, F, \, D_{i,j} \in \bbR^{n \times n}$ are symmetric matrices, $\inprod{X, Y} = \text{trace}(X^T Y)$ denotes the inner product of $X$ and $Y$, and $\mcD_i(X) = (\inprod{D_{i,1}, X}; \ldots; \inprod{D_{i,m_i}, X})$. 

\subsection{Experimental Setup}

We test the problems in
\cref{sec:app} with synthetic and real datasets. The number of samples and features are specified in \cref{tab:exp}. \emph{Synthetic1}
contains samples from a normal distribution, and \emph{Synthetic2} contains samples from a
mixture of $10$ random Gaussians. \emph{Synthetic2} is heterogeneous because
the data block on each individual node is sampled from only $1$ of the $10$
Gaussians.
We also acquire large
empirical datasets from the LIBSVM webpage \cite{liu2009large}, as well as
MNIST digital images \cite{lecun1998gradient}, and CIFAR10 object images
\cite{krizhevsky2009learning}. For binary classification tasks (SVM and logreg), we equally split the $10$ category labels of MNIST and CIFAR into ``positive'' and ``negative'' groups. We use a graph from the \textit{Seventh DIMACS Implementation
Challenge on Semidefinite and Related Optimization Problems}  following
\cite{burer2003nonlinear} for Semidefinite Programming (SDP).  The
regularization parameter is fixed at $\rho=10$ in all experiments.

Consensus ADMM (CADMM) \cite{boyd2011admm}, residual balancing (RB-ADMM)
\cite{he2000alternating}, adaptive ADMM (AADMM) \cite{xu2016adaptive}, and
consensus residual balancing (CRB-ADMM) \cite{song2015fast} are implemented
and reported for comparison. Hyper-parameters of these methods are set as
suggested by their creators. The initial penalty is fixed at $\tau_0 = 1$ for
all methods unless otherwise specified.

\subsection{Convergence results}
\cref{tab:exp} reports the convergence speed in iterations and wall-clock time
(secs) for various test cases. These experiments are performed with
128 cores on a Cray XC-30 supercomputer.  CADMM with default penalty $\tau=1$
\cite{boyd2011admm} is often slow to converge.  ACADMM outperforms the other
ADMM variants on all the real-world datasets, and is competitive with AADMM on
two homogeneous synthetic datasets where the curvature may be globally
estimated with a scalar.

ACADMM is more reliable than AADMM since the curvature estimation becomes
difficult for high dimensional variables. RB is relatively stable but
sometimes has difficulty finding the exact optimal penalty, as the
adaptation can stop because the difference of residuals are not significant
enough to trigger changes. RB does not change the initial penalty in several
experiments such as logistic regression on RCV1.  CRB achieves comparable
results with RB, which suggests that the relative sizes of local residuals may
not always be very informative. ACADMM significantly boosts AADMM and the
local curvature estimations are helpful in practice.

\subsection{Robustness and sensitivity}
\cref{fig:tau} shows that the practical convergence of ADMM is sensitive to the choice of penalty parameter. ACADMM is robust to the selection of the initial penalty parameter and achieves promising results for both homogeneous and heterogeneous data, comparable to ADMM with a fine-tuned penalty parameter.

We study scalability of the method by varying the number of workers and training samples (\cref{fig:enr}). ACADMM is fairly robust to the scaling factor. AADMM occasionally performs well when small numbers of nodes are used, while ACADMM is much more stable. RB and CRB are more stable than AADMM, but cannot compete with ACADMM. \cref{fig:svm} (bottom) presents the acceleration in (wall-clock secs) achieved by increasing the number of workers. 

Finally, ACADMM is insensitive to the safeguarding hyper-parameters,
correlation threshold $\gep^{\text{\scriptsize cor}}$ and convergence constant
$C_{\text{\scriptsize cg}}$. Though tuning these parameters may further
improve the performance, the fixed default values generally perform well in
our experiments and enable ACADMM to run without user oversight. In further
experiments in the supplementary material, we also show that ACADMM is fairly insensitive to the regularization parameter $\rho$ in our classification/regression models.

 
\section{Conclusion}
We propose ACADMM, a fully automated algorithm for distributed
optimization. Numerical experiments on various applications and real-world
datasets demonstrate the efficiency and robustness of ACADMM.  We also prove a
$O(1/k)$ convergence rate for ADMM with adaptive penalties under mild
conditions.
By automating the
selection of algorithm parameters, adaptive methods make distributed systems
more reliable, and more accessible to users that lack expertise in
optimization.


\section*{Acknowledgements} 
ZX , GT, HL and TG were supported by the US Office of Naval Research under grant N00014-17-1-2078 and by the US National Science Foundation (NSF) under grant CCF-1535902. GT was partially supported by the DOD High Performance Computing Modernization Program. MF was partially supported by the Funda\c{c}\~{a}o para a Ci\^{e}ncia e Tecnologia, grant UID/EEA/5008/2013. XY was supported by the General Research Fund from Hong Kong Research Grants Council under grant HKBU-12313516. 


\nocite{studer2014democratic}
\nocite{goldstein2010high}

\bibliography{admm}
\bibliographystyle{icml2017}

\end{document}